\documentclass[journal]{IEEEtran}

\usepackage{mathtools, amsthm, amsopn, amsmath}
\usepackage{etex}
\usepackage{graphicx}
\usepackage{endnotes}
\usepackage{hyperref}
\usepackage{epsfig,psfrag}
\usepackage{pst-all}
\usepackage{amssymb,amsfonts,upref,cite,epsf,color,bm}
\usepackage{graphicx}
\usepackage{color}
\usepackage{calc}
\usepackage{booktabs}
\usepackage{tikz}
\usepackage{pgfplots}
\newcommand\defeq{:=}
\usepackage{algorithm, float}% http://ctan.org/pkg/Alg.\s
\usepackage{algpseudocode}% http://ctan.org/pkg/Alg.\icx
\floatname{algoirthm}{algorithm}
\algnewcommand\algorithmicinput{\textbf{Input:}}
\algnewcommand\INPUT{\item[\Alg.\icinput]}
\algnewcommand\algoirhtmicoutput{\textbf{Output:}}
\algnewcommand\OUTPUT{\item[\algorithmicoutput]}

\DeclareMathOperator*{\argmin}{argmin}

\newcommand{\vx}{x}

\newcommand{\partition}{\mathcal{F}}
\newcommand{\cluster}{\mathcal{C}}
\newcommand{\samplingset}{\mathcal{M}}
\newcommand{\edges}{\mathcal{E}}
\newcommand{\nodes}{\mathcal{V}}
\newcommand{\graph}{\mathcal{G}}
\newcommand{\graphsigs}{\mathbb{R}^{\mathcal{V}}}

\newtheorem{theorem}{Theorem}%[section]

\newtheorem{lemma}[theorem]{Lemma}

\begin{document}

\title{The Logistic Network Lasso} %for Semi-Supervised Learning from Big Data over Networks}

\author{Henrik~Ambos,
        Nguyen~Tran,
        and~Alexander~Jung
\thanks{Authors are with the Department of Computer Science, Aalto University, Finland; firstname.lastname(at)aalto.fi}
}

%\address{\normalsize Department of Computer Science, Aalto University, Finland; firstname.lastname(at)aalto.fi\\[-0.5mm] }

%\thanks{\hspace*{-5mm}The work of ??? was supported by ???.} 
	\maketitle
%  -----------------------  Abstract  ----------------------------
\begin{abstract}
	
	We apply the network Lasso to solve binary classification and clustering problems for network-structured data. 
	To this end, we generalize ordinary logistic regression to non-Euclidean data with an intrinsic network 
	structure. The resulting ``logistic network Lasso'' amounts to solving a non-smooth convex 
	regularized empirical risk minimization. The risk is measured using the logistic loss incurred over a small 
	set of labeled nodes. For the regularization we propose to use the total variation of the classifier requiring 
	it to conform to the underlying network structure. A scalable implementation the learning method is obtained 
	using an inexact variant of the alternating direction methods of multipliers which results in a scalable learning algorithm.
	
\end{abstract}

\begin{IEEEkeywords}
Lasso, big data over networks, semi-supervised learning, classification, clustering, complex networks, convex optimization, ADMM
\end{IEEEkeywords}

\IEEEpeerreviewmaketitle

\section{Introduction}
        
        The recently introduced extension of the \emph{least absolute shrinkage and selection operator} (Lasso) 
	to network-structured data, coined the ``network Lasso'' (nLasso), allows efficient processing of massive 
	datasets using modern convex optimization methods \cite{NetworkLasso}. In this paper, we apply nLasso 
	to semi-supervised classification and clustering of massive network-structured datasets (big data over networks) \cite{Bhagat11,LargeGraphsLovasz,BigDataNetworksBook}. 
	
%	Network-structured datasets arise in a wide range of application domains ranging from image- and video processing, 
%	experimental physics to social networks \cite{SandrMoura2014}. The acquisition of label information for those datasets 
%	is typically costly as it requires manual labor. In order to cope with limited amounts of labeled data nLasso exploits the network structure 
%	by enforcing labels of well-connected subsets of nodes (cluster) to be similar.  %The classification is semi-supervised in the sense 
%	%of only having access to the labels of few data points. 

	Most of the existing work on nLasso-based methods focused on predicting numeric labels (or target variables) within regression 
	problems \cite{WhenIsNLASSO,NNSPFrontiers,NetworkLasso,Chen2015,SandrMoura2014,JungSLPcomplexit2018}. 
	In contrast, we apply nLasso to binary classification (and clustering) problems which assign binary-valued 
	labels to data points. The resulting logistic nLasso (lnLasso) aims at balancing the empirical error, measured 
	using the logistic loss, incurred for a small number of data points whose labels are known with the amount 
	by which the resulting classifier conforms to the underlying network structure. 
	
	%In order to learn a classifier from partially labeled data, we employ the logistic loss function to measure the 
	%empirical (or training) error incurred by a particular classifier. Moreover, we learn classifiers which conform to the 
	%underlying network structure. 
	Thus, lnLasso is an instance of regularized empirical risk minimization with the total variation 
	of the classifier as regularization term \cite{VapnikBook}. This minimization problem is a non-smooth convex optimization problem 
	which we solve using an inexact form of the alternating direction method of multipliers (ADMM) \cite{DistrOptStatistLearningADMM}.
	 
	The lnLasso provides an alternative to the family of label propagation (LP) methods \cite{SemiSupervisedBook}. While LP methods 
	are based on using the squared error loss to measure the empirical risk incurred over labeled data points, the lnLasso uses the 
	average logistic loss over the labeled data points to assess the quality of a particular classifier. 
	
	While lnLasso is based on a probabilistic model for the labels, it considers the network structure as given and fixed. This is different 
	from the semi-supervised classification method presented in \cite{Zhan2014}. In particular, \cite{Zhan2014} applies belief propagation 
	for a stochastic block model (SBM) to semi-supervised classification. This mehtod is an approximation to the Bayes optimal classifier given 
	the probabilistic SBM. 
	
	%The setting considered in this paper is also close to \cite{Zhan2014} which derives a classification method from a probabilistic 
	%modelling perspective using the stochastic block model (SBM). 
	Finally, our method is closely related to graph-cut methods \cite{Ruusuvuori2012,Kechichian2018,Boykov2004}. Indeed, both 
	are based on a similar optimization problem. In contrast to graph-cut methods, our approach provides a precise probabilistic i
	interpretation of this optimization problem. Moreover, while our implementation of lnLasso is based on convex optimization 
	(allowing for highly scalable implementations), graph-cuts is based on combinatorial optimization. 
	
%\vspace{-1mm}	
\section{Problem Formulation}
\label{sec_problem_formuation} 
%\vspace{-2mm}	

	We consider network-structured datasets that can be represented by an undirected weighted graph (the 
	``empirical graph'') $\graph = (\nodes, \edges, \mathbf{W})$. A particular node $i \in \nodes$ of the graph represents an 
	individual data point (such as a document, an image or an entire social network user profile). Two different data points $i,j \in \nodes$ 
	are connected by undirected edges $\{i,j\} \in \edges$ if these data points are considered similar (such as the documents authored by the same person or 
	the social network profiles of befriended users). 
	
	Each undirected edge $\{i,j\} \in \edges$ is associated with a positive weight $W_{ij}>0$ 
	quantifying the amount of similarity between data points $i,j \in \nodes$. The neighbourhood of a node $i \in \nodes$ is defined as 
	$\mathcal{N}(i) \defeq \{ j : \{i,j\} \in \edges \}$. Without essential loss of generality, 
	we only consider datasets with a connected empirical graph $\graph$ with more than one node. 
	
	It will be useful to think of an undirected edge $\{i,j\}\!\in\!\edges$ as a pair of two directed edges $(i,j)$ and $(j,i)$. 
	With a slight abuse of notation we denote by $\edges$ the set of undirected edges as well as the set of directed edges obtained 
	by replacing each undirected edge by a pair of directed edges. %We highlight 
	%that the particular choice of orientation for the empirical graph $\graph$ has no effect on our results and methods and will only be used for notational convenience. % orientation for the graph $\graph$. 
	
	On top of the network structure, datasets often convey additional information such as labels $y_{i}$ (e.g., the class membership) 
	of the data points $i \in \nodes$. %In a social network application, the labels $y_{i}$ might encode the membership of user $i \in \nodes$ to a particular social group. 
	In what follows, we focus on binary classification problems involving binary labels $y_{i} \in \{-1,1\}$. 
	Since the acquisition of reliable label information is costly, we typically have access only to  
	few labels $y_{i}$ of the nodes in a small sampling set $\samplingset\!\subset\!\nodes$.
	
	%The  are typcially unkown for most data points $i \in \nodes$, 
	We model the labels $y_{i}$ of the data points $i \in \nodes$ as independent random variables with (unknown) 
	probabilities 
     \begin{equation} 
	\label{equ_def_p_i}
	p_{i} \defeq {\rm Prob} \{ y_{i} \!=\! 1 \}, 
	\end{equation} 
	We parametrize these probabilities using  the logarithm of the 
	``odds ratio'',
	%particular we define the log odds ratio associated with node $i \in \nodes$
	\begin{equation}
	\label{equ_def_log_odds_ratio}
		\vx[i] \defeq {\rm log} (p_{i} / (1-p_{i})).
	\end{equation}	
	
	The quantity \eqref{equ_def_log_odds_ratio} defines a graph signal $\vx[\cdot]: \nodes \rightarrow \mathbb{R}$  
	assigning each node $i \in \nodes$ of the empirical graph $\graph$ the signal value $x[i]$. 
	Our approach uses a graph signal $x[\cdot]$ to represent a classifier for the data points in $\graph$. 
	In particular, we classify a data point $i \in \nodes$ as $\hat{y}_{i}={\rm sign} \{ x[i]\}$. 
	Any reasonable classifier $x[\cdot]$ should agree well with available label information 
	such that $ \hat{y}_{i} \approx  y_{i}$ for all $i \in \samplingset$. 
	
	\vspace*{-2mm}
	\section{Logistic Network Lasso}
\label{sec_lNLasso}
	
	It is sensible to learn a classifier $x[\cdot]: \nodes \rightarrow \mathbb{R}$ from few initial labels 
	$\{ y_{i} \}_{i \in \samplingset}$ by maximizing their probability (evidence), or equivalently by minimizing 
	the empirical error  
	\begin{equation} 
	\label{equ_def_emp_risk}
	\widehat{E}(\vx[\cdot])\!\defeq\!\frac{1}{\left|\samplingset \right|} \sum_{i \in \samplingset} \!\ell ( y_{i} \vx[i]) 
	\end{equation} 
	with the logistic loss function 
		\vspace*{-2mm}
	\begin{equation}
	\label{equ_def_log_loss}
	\ell (z) \defeq {\rm log} (1 + \exp(-z) ).
 \vspace*{-1mm}
	\end{equation}
	
	The criterion \eqref{equ_def_emp_risk} by itself is not sufficient for guiding the learning of a classifier $x[\cdot]$ 
	based on the labels $\{ y_{i} \}_{i \in \samplingset}$. Indeed, the criterion \eqref{equ_def_emp_risk} 
	ignores the signal values $x[i]$ at non-sampled nodes $i \in \nodes \setminus \samplingset$. 
	
	In order to learn an entire classifier $x[\cdot]$ from the incomplete information provided 
	by the initial labels $\{ y_{i} \}_{i \in \samplingset}$, we need to impose some additional structure on the classifier $x[\cdot]$. 
	In particular, any reasonable classifier $x[\cdot]$ should conform with the \emph{cluster structure} of the empirical graph $\graph$ \cite{NewmannBook}. 
	
	The extend to which a graph signal $x[\cdot]$ conforms with the cluster structure is measured by the total variation (TV)
	\begin{align}
	\label{equ_def_TV_norm}
		\| \vx[\cdot] \|_{\rm TV} & \defeq \sum_{\{i,j\}\in \edges} W_{ij} \left| \vx[j] - \vx[i] \right|.   \\[-4mm] % \nonumber \\
		%& = (1/2) \sum_{(i,j)\in \edges} W_{ij} \left| \vx[j] - \vx[i] \right| .  \\[-4mm]
		\nonumber
	\end{align}
	Indeed, a graph signal $x[\cdot]$ has a small TV only if the signal values $x[i]$ are approximately constant 
	over well connected subsets (clusters) of nodes. %This ``clustering hypothesis'' (or variations thereof) motivates many 
%	methods for (semi-) supervised learning \cite{SemiSupervisedBook}. 
	
	We are led quite naturally to learning a classifier $x[\cdot]$ %by 
	%balancing a small empirical error (risk) $\widehat{E}(x[\cdot])$ (cf. \eqref{equ_def_emp_risk}) with a small TV $\| \vx[\cdot] \|_{\rm TV}$ (cf. \eqref{equ_def_TV_norm}). 
	via the following \emph{regularized empirical risk minimization} 
	\begin{align} \label{optProb}
		\hat{\vx}[\cdot] & \in \argmin_{x[\cdot] \in \graphsigs} \widehat{E}(\vx[\cdot])  + \lambda \| \vx[\cdot] \|_{\rm TV}. 
	\end{align}
	The regularization parameter $\lambda$ in \eqref{optProb} allows to trade-off a small TV  $\| \hat{x}[\cdot] \|_{\rm TV}$ 
	of the classifier $\hat{x}[\cdot]$ against a small empirical error $\widehat{E}(\hat{x}[\cdot])$ (cf. \eqref{equ_def_emp_risk}). 
	We refer to \eqref{optProb} as the logistic nLasso (lnLasso) problem. The choice of $\lambda$ can be guided by cross 
	validation techniques \cite{hastie01statisticallearning}. 
	
	Note that lnLasso \eqref{optProb} does not enforce directly the labels $y_{i}$ 
	to be clustered. Instead, it requires the classifier $x[\cdot]$ (which parametrizes the probability distributed of 
	the labels $y_{i}$ via \eqref{equ_def_log_odds_ratio}) to be clustered (have a small TV). 
	%, i.e., a small average logistic loss incurred for the labelled data points $i \in \samplingset$. 
	%In particular, choosing a small value of $\lambda$ will result in a classifier $\hat{x}[\cdot]$ with a small empirical 
	%error , while choosing a large value of $\lambda$ favours 
	%a classifier $\hat{x}[\cdot]$ with a small TV.

\section{Implementation via Inexact ADMM}
\label{sec_lNLasso_ADMM}

        We now present a classification method which is obtained by solving \eqref{optProb} using 
        inexact ADMM. To this end, we introduce for each directed edge $(i,j) \in \edges$ 
        of the oriented empirical graph $\graph$, the auxiliary variable $z_{ij}$. These variables  
	act as local copies of the optimization variables $x[i]$ in \eqref{optProb}. 
       
         We can then reformulate lnLasso \eqref{optProb} as (cf.\ \eqref{equ_def_TV_norm})
	\begin{align}\label{LNLprob}
		\hat{\vx}[\cdot] \in &\argmin_{x[\cdot] \in \graphsigs}  \widehat{E}(x[\cdot]) + (\lambda/2) \sum_{ ( i, j ) \in \edges } W_{i,j} |z_{ij} - z_{ji} | \\
		&\text{s.t.} \indent x[i] = z_{ij}\indent i\in\nodes, \indent j\in\mathcal{N}(i).  \label{equ_local_consistency}
	\end{align}
	The reformulation \eqref{LNLprob} of the lnLasso \eqref{optProb} is computationally appealing since the objective function 
	in \eqref{LNLprob} consists of two independent terms. The first term is the empirical risk $\widehat{E}(x[\cdot])$ which measures how well the initial labels $y_{i}$ agree 
	with the classifier $x[\cdot]$. The second term is the scaled TV $\lambda \sum_{ (i, j) \in \edges } W_{i,j} |z_{ij} - z_{ji} |$, which measures 
	how well the classifier $x[\cdot]$ is aligned with the cluster structure of $\graph$. 
        These two terms are coupled via \eqref{equ_local_consistency}. 

	In order to solve the non-smooth convex optimization problem \eqref{LNLprob}, we apply an 
	inexact variant of ADMM \cite{EcksteinYao2017}. To this end, we define the augmented Lagrangian \cite{DistrOptStatistLearningADMM}
	\begin{align}
	\label{nLasso}
		\mathcal{L}(x[\cdot],z_{ij},u_{ij}) &\defeq \widehat{E}(x[\cdot]) + (\lambda/2) \sum_{(i,j) \in \edges} W_{ij} \left| z_{ij} - z_{ji} \right|  \\[-4mm]
		&\hspace*{-15mm}+\!(\rho/2) \hspace*{-3mm} \sum_{(i,j) \in \edges}  \hspace*{-3mm}\big[ (x[i]\!-\!z_{ij}\!+\!u_{ij})^2\!-\!u_{ij}^{2} %\!+\!  (x[j]\!-\!z_{ji}\!+\!u_{ji})^2\!-\!u_{ji}^{2}
		\big]\nonumber\\[-7mm]
		\nonumber
	\end{align}
	with dual variables $u_{ij}$ introduced for each edge $(i,j) \in \edges$.
	
	Ordinary (exact) ADMM optimizes $\mathcal{L}(x[\cdot],z_{ij},u_{ij})$ block coordinate-wise by iterating the following updates:
		\vspace*{-1mm}
	\begin{align}
		\hspace *{-2mm} x^{(k+1)}[\cdot] &\!\defeq\! \argmin_{x[\cdot] \in \graphsigs}  \mathcal{L}(x[\cdot],z_{ij}^{(k)},u_{ij}^{(k)})  \label{equ_x_update} \\[1mm]
		z^{(k+1)}_{ij}  &\!\defeq\! \argmin_{ z_{ij}  \in \mathbb{R}}  \mathcal{L}(x^{(k+1)}[\cdot], z_{ij} , u_{ij}^{(k)}) \label{equ_z_update} \\[1mm]
		u^{(k+1)}_{ij} & \hspace*{-2mm}\!\defeq\! u^{(k)}_{ij}\!+\!x^{(k+1)}[i]\!-\!z^{(k+1)}_{ij} \mbox{ for each } (i,j) \!\in\! \edges. \label{equ_u_update}\\[-4mm]
		\nonumber
	\end{align}
	The update \eqref{equ_x_update} minimizes the empirical error $\widehat{E}(\vx[\cdot])$, while update \eqref{equ_z_update} 
	minimizes the TV $\| \vx[\cdot] \|_{\rm TV}$. These two minimization processes are coupled via \eqref{equ_u_update} using 
	the dual variables $\{ u_{ij} \}_{(i,j) \in \edges}$. Each dual variable $u_{i,j}$ corresponds to a particular directed edge $(i,j) \in \edges$ 
	and the corresponding constraint $x[i] =z_{ij}$ (see \eqref{equ_local_consistency}). 
	
	Using the masked labels $\tilde{y}_{i} \!=\! y_{i}\!\in\!\{-1,1\}$ for sampled nodes $i \in \samplingset$ and $\tilde{y}_{i} = 0$ otherwise, 
	the  update \eqref{equ_x_update} becomes  
	\begin{align}
	\label{updateX}
		\hspace*{-3mm}x^{(k\!+\!1)}[i] \!=\! \argmin_{x \in \mathbb{R}} \underbrace{ \ell( \tilde{y}_{i} x)\!+\!\frac{|\samplingset|\rho}{2} \hspace*{-2mm} \sum_{j \in \mathcal{N}(i)}\hspace*{-2mm} ( x\!-\!z_{ij}^{(k)}\!+\!u_{ij}^{(k)})^2}_{\defeq f_{i}(x)}.
	\end{align}
	
	The update \eqref{equ_z_update} can be worked out as (see, e.g., \cite{NetworkLasso})
	\begin{align}
		\hspace*{-9mm}{z}_{ij}^{(k+1)} \!=\! \theta ({x}^{(k\!+\!1)}[i]\!+\!{u}_{ij}^{(k)})\!+\!(1\!-\!\theta) ({x}^{(k+1)}[j]\!+\!{u}_{ji}^{(k)}) \nonumber 
		%{z}_{ji}^{(k+1)} &\!=\! (1\!-\!\theta) ({x}^{k+1}[i]\!+\!{u}_{ij}^{(k)})\!+\!\theta ({x}^{(k+1)}[j]\!+\!{u}_{ji}^{(k)}), \nonumber 
\end{align} 
with 
\begin{equation} 
	\theta \!=\! \max \left(\frac{1}{2}, 1-\frac{(\lambda/\rho) W_{ij}}{ |{x}^{(k+1)}[i]\!+\!{u}_{ij}^{(k)}\!-\!{x}^{(k+1)}[j]\!-\!{u}_{ji}^{(k)}|} \right). \nonumber 
\end{equation}

	The presence of the logistic loss function \eqref{equ_def_log_loss} precludes a closed-form solution of \eqref{updateX}. 
	However, since \eqref{updateX} amounts to a scalar unconstrained smooth minimization, we can solve \eqref{updateX} approximately 
	by cheap iterative methods.  
	%allowing for efficient numerical methods to approximately solve \eqref{updateX}. 
	Replacing the ADMM update \eqref{updateX} by an inexact update might 
	still yield convergence to the solution of \eqref{LNLprob} \cite[Theorem 8]{EcksteinBert92}. 

	As we will show, \eqref{updateX} can be approximated by 
	\begin{equation} 
	\label{equ_inexact_update}
         \hat{x}^{(k\!+\!1)}[i] \!= \hspace*{-6mm} \underbrace{\Phi_{i} \circ \ldots \circ \Phi_{i}}_{|\tilde{y}_{i}|\lceil 2\log(2(k\!+\!1)) / \log ( |\samplingset| \rho d_{i}) \rceil} \hspace*{-3mm}(1/d_{i}) \sum_{j \in \mathcal{N}(i)} \hspace*{-2mm} \big( z^{(k)}_{ij} \!-\! u^{(k)}_{ij} \big)   
	\end{equation} 
	with the node degree $d_{i} = |\mathcal{N}(i)|$ and the map 
	\begin{equation} 
	\label{equ_phi_i}
	\Phi_{i} (x)\!\defeq\!\frac{\tilde{y}_{i}/(|\samplingset|d_{i}\rho)}{1\!+\!\exp(\tilde{y}_{i}x)}\!+\!(1/d_{i}) \hspace*{-2mm} \sum_{j \in \mathcal{N}(i)} \hspace*{-2mm}\big( z^{(k)}_{ij} \!-\! u^{(k)}_{ij} \big).
	\end{equation} 
         Replacing the update \eqref{updateX} with \eqref{equ_inexact_update} yields Alg. \ref{alg:ADMM}. 
         	\begin{algorithm}[]
		\caption{lnLasso via inexact ADMM}\label{alg:ADMM}
\begin{algorithmic}[1]
\renewcommand{\algorithmicrequire}{\textbf{Input:}}
\renewcommand{\algorithmicensure}{\textbf{Output:}}
\Require   $\graph$, $\samplingset$, $\{ y_{i} \}_{i \in \samplingset}$, $\lambda$, $\rho$
\Statex\hspace{-6mm}{\bf Initialize:}$k\!\defeq\!0$, $x^{(0)}[i]\!\defeq\!0$, $z^{(0)}_{ij}\!\defeq\!0$, $u^{(0)}_{ij}\!\defeq\!0$, $\tilde{y}_{i}\!\defeq\!0$, $\tilde{y}_{i} \defeq y_{i}$ for all sampled nodes $i \in \samplingset$
		%\begin{Alg.\ic}[]
		%	\Statex 
		%	\State initialize 
			\Repeat
				\For{$i \in \nodes$}
					%\If {$i \in \samplingset$}
					      \State set $\tilde{x}^{(0)}\!\defeq\!(1/d_{i}) \sum\limits_{j \in \mathcal{N}(i)} (z_{ij}^{(k)}\!-\!u_{ij}^{(k)})$
					      	\vspace*{2mm}
					       \State set $n \defeq 0$
					       	\vspace*{2mm}
					    %  \State set $c \defeq  M \rho d_{i}$, $g 
					       \Repeat 
					        \State $ \tilde{x}^{(n+1)} \defeq \Phi_{i}(\tilde{x}^{(n)})$  (see \eqref{equ_phi_i})
					        	\vspace*{2mm}
						\State $ n \defeq n +1$
							\vspace*{2mm}
						\Until $n \geq \lceil |\tilde{y}_{i}| 2 \log(2(k\!+\!1)) / \log ( |\samplingset| \rho d_{i}) \rceil$
							\vspace*{2mm}
						 \State $\hat{x}^{(k+1)}[i] \defeq  \tilde{x}^{(n)}$ 
%					\Else
%						\State  $x^{(k+1)}[i] \defeq  (1/d_{i}) \sum\limits_{j \in \mathcal{N}(i)} (u_{ij}^{(k)}\!-\!z_{ij}^{(k)})$  
%					\EndIf
				\EndFor
				\hspace*{-20mm}	\For{ $(i,j) \in \edges$}
					\vspace*{2mm}
					\State \hspace*{-10mm}$\theta \defeq \max \bigg\{ 1/2, 1\!-\!\frac{(\lambda/\rho) W_{ij}}{ |\hat{x}^{(k\!+\!1)}[i]\!+\!{u}^{(k)}_{ij}\!-\!\hat{x}^{(k\!+\!1)}[j]\!-\!{u}^{(k)}_{ji} |} \bigg\}$
					\vspace*{2mm}
					\State \hspace*{-10mm} ${z}_{ij}^{(k\!+\!1)}\!\defeq\!\theta (\hat{x}^{(k\!+\!1)}[i]\!+\!{u}_{ij}^{(k)})\!+\!(1\!-\!\theta) (\hat{x}^{(k\!+\!1)}[j]\!+\!{u}_{ji}^{(k)})$
					\vspace*{2mm}
				%	\State \hspace*{-10mm}  ${z}_{ji}^{(k+1)}\!\defeq\!(1\!-\!\theta) ({x}^{(k+1)}[i]\!+\!{u}_{ij}^{(k)}) + \theta ({x}^{(k\!+\!1)}[j]\!+\!{u}_{ji}^{(k)})$		
			%		\vspace*{2mm}
					\State \hspace*{-10mm} ${u}_{ij}^{(k\!+\!1)} \defeq {u}_{ij}^{(k)}\!+\!(\hat{x}^{(k\!+\!1)}[i]\!-\! {z}_{ij}^{(k\!+\!1)} )$
					\vspace*{2mm}
					%\State \hspace*{-10mm}  ${u}_{ji}^{(k+1)} \defeq {u}_{ji}^{(k)}\!+\!({x}^{(k+1)}[j]\!-\!{z}_{ji}^{(k+1)} )$
					%\vspace*{2mm}
				\EndFor
					\vspace*{2mm}
			\State  $k\!\defeq\!k\!+\!1$
				\vspace*{2mm}
			\Until convergence 
				\vspace*{2mm}
			\Ensure classifier $\hat{x}[i] \defeq \hat{x}^{(k)}[i]$ \mbox{ for all }$i \in \nodes$
		\end{algorithmic}
	\end{algorithm}
	
       The classifier $\hat{x}[i]$ delivered by Alg.\ \ref{alg:ADMM} is then used to label the data points as $\hat{y}_{i} = 1$ 
       if $\hat{x}[i] > 0$ and $\hat{y}_{i} = -1$ otherwise. However, the classifier $\hat{x}[i]$ provides more information than 
       just the resulting (predicted) labels $\hat{y}_{i}$. 
       
       Indeed, the magnitudes $|\hat{x}[i]|$ quantify the confidence in the predicted labels $y_{i}$. A magnitude 
       $|\hat{x}[i]|$ close to zero indicates the predicted label $\hat{y}_{i}$ to be unreliable. On the other 
       hand, if the magnitude $|\hat{x}[i]|$ is large then we can be quite confident in the predicted label $\hat{y}_{i}$. 

       The convergence of the iterates $\hat{x}^{(k)}[i]$ generated by Alg.\ \ref{alg:ADMM} to the 
       solution $\hat{x}[i]$ of the lnLasso problem \eqref{LNLprob} can be verified from \cite[Theorem 8]{EcksteinBert92}. 
       In particular, convergence is guaranteed (for any $\rho>0$) 
       if the errors $\varepsilon_{k} = |\hat{x}^{(k+1)}[i] - x^{(k+1)}[i]|$ are sufficiently small such that $\sum_{k=1}^{\infty} \varepsilon_{k} < \infty$. 
       The following result verifies 
       exactly this condition. 
       \begin{lemma} 	
       \label{lem_approximate_update}
	For $|\samplingset| \rho d_{i} >1$, the deviation between the inexact and exact updates 
	\eqref{equ_inexact_update} and \eqref{updateX} satisfies 
	\begin{equation} 
	\label{equ_bound_lemma}
	| \hat{x}^{(k\!+\!1)}[i]\!-\!x^{(k\!+\!1)}[i]| \leq 1/(k\!+\!1)^2. 
	\end{equation}
	\end{lemma} 
	\begin{proof} 
	The update \eqref{updateX} is an unconstrained minimization of a differentiable 
	convex function $f_{i}(x)$. The solutions $x_{0}$ of \eqref{updateX} are solutions of 
	$f'_{i}(x_{0}) = 0$ \cite{BoydConvexBook}. Working out the derivative $f'_{i}(x)$, this ``zero-gradient 
	condition'' becomes  
	\begin{equation} 
	\label{equ_opt_condition_update_approx}
	\frac{-\tilde{y}_{i}}{1\!+\!\exp(\tilde{y}_{i}x_{0})}\!+\! |\samplingset| \rho \sum_{j \in \mathcal{N}(i)}(x_{0}\!-\!z_{ij}^{(k)}\!+\!u_{ij}^{(k)}) = 0.
	\end{equation}  
	%as necessary and sufficient for $x_{0}$ to solve \eqref{updateX}. 
	The necessary and sufficient condition \eqref{equ_opt_condition_update_approx} for $x_{0}$ to solve \eqref{updateX} is, in turn, 
	equivalent to the fixed-point characterization 
	\vspace*{-2mm}
	\begin{equation}
	\label{equ_fixed_point_update}
        x_{0} = \Phi_{i} (x_{0}), 
        \vspace*{-2mm}
	\end{equation} 
	with the map $\Phi_{i}$ defined in \eqref{equ_phi_i}. 
	
	The approximate update \eqref{equ_inexact_update} is a fixed-point iteration \cite{BausckeCombette,JungFixedPoint} 
	for solving \eqref{equ_fixed_point_update}. The map $\Phi_{i}$ \eqref{equ_phi_i} is a contraction over the interval 
	$[b_{i}\!-\!1,b_{i}\!+\!1]$ with the shorthand $b_{i} \defeq (1/d_{i}) \sum_{j \in \mathcal{N}(i)}(z_{ij}^{(k)}\!-\!u_{ij}^{(k)})$. 
	In particular, 
	\vspace*{0mm}
         \begin{equation}
         \nonumber
         \Phi_{i} ([b_{i}\!-\!1,b_{i}\!+\!1]) \subseteq [b_{i}\!-\!1,b_{i}\!+\!1]
         \vspace*{-2mm}
         \end{equation}
         and 
         \vspace*{0mm}
         \begin{equation}
         \nonumber
         |\Phi_{i}(x)\!-\!\Phi_{i}(x')|\!\leq\! |x\!-\!x'|/(|\samplingset|\rho d_{i}) \mbox{ for } x,x'\!\in\![b_{i}\!-\!1,b_{i}\!+\!1].
         \vspace*{0mm}
         \end{equation} 
	The bound \eqref{equ_bound_lemma} follows then from standard results on fixed-point iterations 
	(see, e.g., \cite[Thm. 1.48]{BausckeCombette}). 
	\end{proof}

         Note that Alg.\ \ref{alg:ADMM} is highly scalable as it can be implemented using message passing over the 
	empirical graph $\graph_{\rm syn}$. 
	%with each node $i \in \nodes$ being equipped with a computational unit which 
	%updates $\hat{x}^{(k+1)}[i]$ based on local information provided by the variables $\{ \hat{x}^{(k)}[j], z^{(k)}_{ij}, u^{(k)}_{ij}, z^{(k)}_{ji}, u^{(k)}_{ji} \}_{j \in \mathcal{N}(i)}$. 

%\vspace*{-2mm}
\section{Numerical Experiments}
%\vspace*{-1mm}
\label{sec_num}

	We assess the performance of lnLasso (Alg.\ \ref{alg:ADMM}) by means of numerical experiments using a synthetic dataset 
	whose network structure conforms to the stochastic block model (SBM). %with the classification method presented in \cite{Zhan2014}, 
	%which is based on approximate Bayesian inference via belief propagation (BP) using a SBM. 
        In particular, we consider a dataset (of size $N=500$) whose empirical graph $\graph$ is generated according to a SBM 
        with two clusters $\cluster_{1}$, $\cluster_{2}$. The nodes are assigned to these two clusters (blocks) randomly with equal 
        probability $1/2$. Two nodes within the same group are connected with probability 
        $c_{\rm in}/N$ while nodes of different groups are connected with probability $c_{\rm out}/N$. 

        	The weights of edges connecting nodes within cluster $\cluster_{1}$ are chosen as $W_{i,j} \sim [\mathcal{N}(10,1)]_{+}$ 
	and the weights for edges within cluster $\cluster_{2}$ as $W_{i,j} \sim [\mathcal{N}(12,1)]_{+}$.
	The weights of boundary edges $\{i,j\} \in \partial\partition$ are drawn according to $W_{i,j} \sim [\mathcal{N}(3,1)]_{+}$.  
	The true underlying labels $y_{i} = 1$ for $i \in \cluster_{1}$ and $y_{i} =-1$ for $i \in \cluster_{2}$ are observed only for 
	nodes in the sampling set $\samplingset$ selected uniformly at random. 
	
	The experiments involve $K\!=\!1000$ i.i.d.\ realizations of the empirical graph $\graph$ and sampling set $\samplingset$. For each 
	realization of $\graph$ and $\samplingset$, we execute lnLasso with $\lambda=2 \cdot 10^{-5}$ and $\rho=1$ 
	(which implies the condition $|\samplingset| d_{i} \rho>1$ required by Lemma \ref{lem_approximate_update}). 
	For comparison, we also implemented belief propagation (BP) for SBM \cite{Zhan2014}, 
	 plain vanilla LP \cite{Zhu02learningfrom} and the max-flow (graph-cut) method \cite{Boykov2004}. 
	
	In Fig.\ \ref{fig:epsilon_test}, we plot the average classification accuracies (rate of correct labels) achieved by 
	the various methods for varying ratio $\epsilon\!=\!c_{\rm in}/c_{\rm out}$ of the SBM used to generate the empirical graph. 
	Choosing $\epsilon \approx 1$ corresponds to a weak cluster structure, while for $\epsilon\!\gg\!1$ the cluster structure 
	is strongly pronounced. The shaded areas in Fig.\ \ref{fig:epsilon_test} indicate 
	the empirical standard deviation obtained over the $1000$ simulation runs. 
	\vspace*{0mm}
		\begin{figure}[htbp]
		\includegraphics[width=0.9\columnwidth]{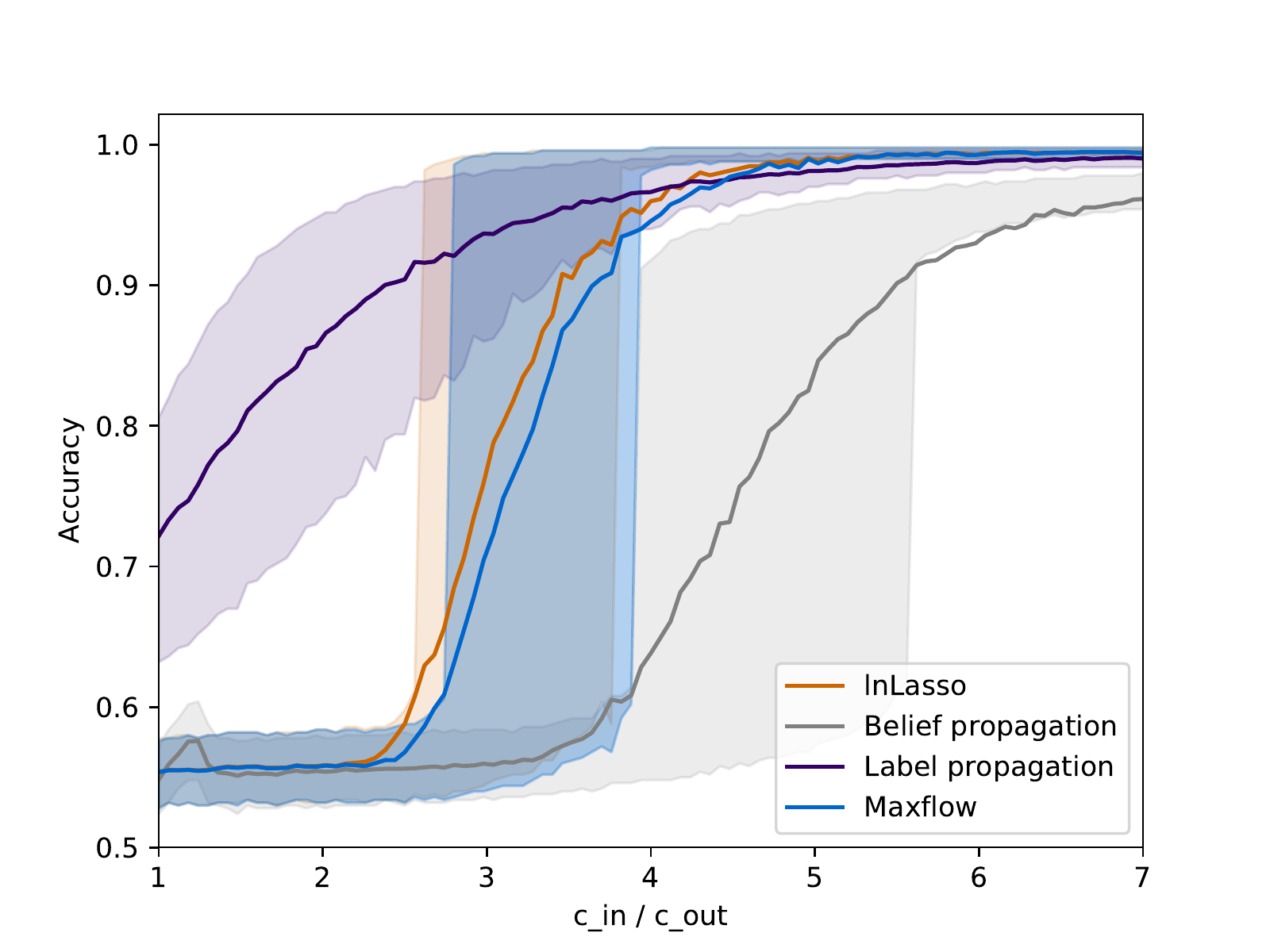}
		\vspace*{-3mm}
		\caption{Classification accuracy for varying $c_{\rm in} / c_{\rm out}$.}
		\label{fig:epsilon_test}
	\end{figure}

	According to Fig.\ \ref{fig:epsilon_test}, lnLasso performs poor relative to LP for 
	$\epsilon\!<\!2.5$, while for $\epsilon\!>\!4$, lnLasso (along with the max-flow method) slightly outperforms LP. 
	Moreover, our results suggest that lnLasso is more robust regarding the variations in the precise network structure compared to BP.

	In Fig.\ \ref{fig:alpha_test}, for the SBM with $\epsilon\!=\!5$, we depict the classification accuracy as a function of the 
	amount of labeled nodes (selected uniformly at random) quantified by the sampling ratio $\alpha\!=\!|\mathcal{M}| /N$. 
	The solid lines represent the average accuracies obtained for $1000$ simulation runs.
	For small sampling sets ($\alpha \ll1 $) the classification accuracy of lnLasso is poor compared to LP. 
	However, for $\alpha\!>\!1/20$ lnLasso becomes significantly 
	more accurate and for $\alpha\!>\!1/10$ even slightly outperforms LP. Overall, 
	lnLasso is slightly better than max-flow and significantly better than BP.

	\begin{figure}[htbp]
		\includegraphics[width=0.9\columnwidth]{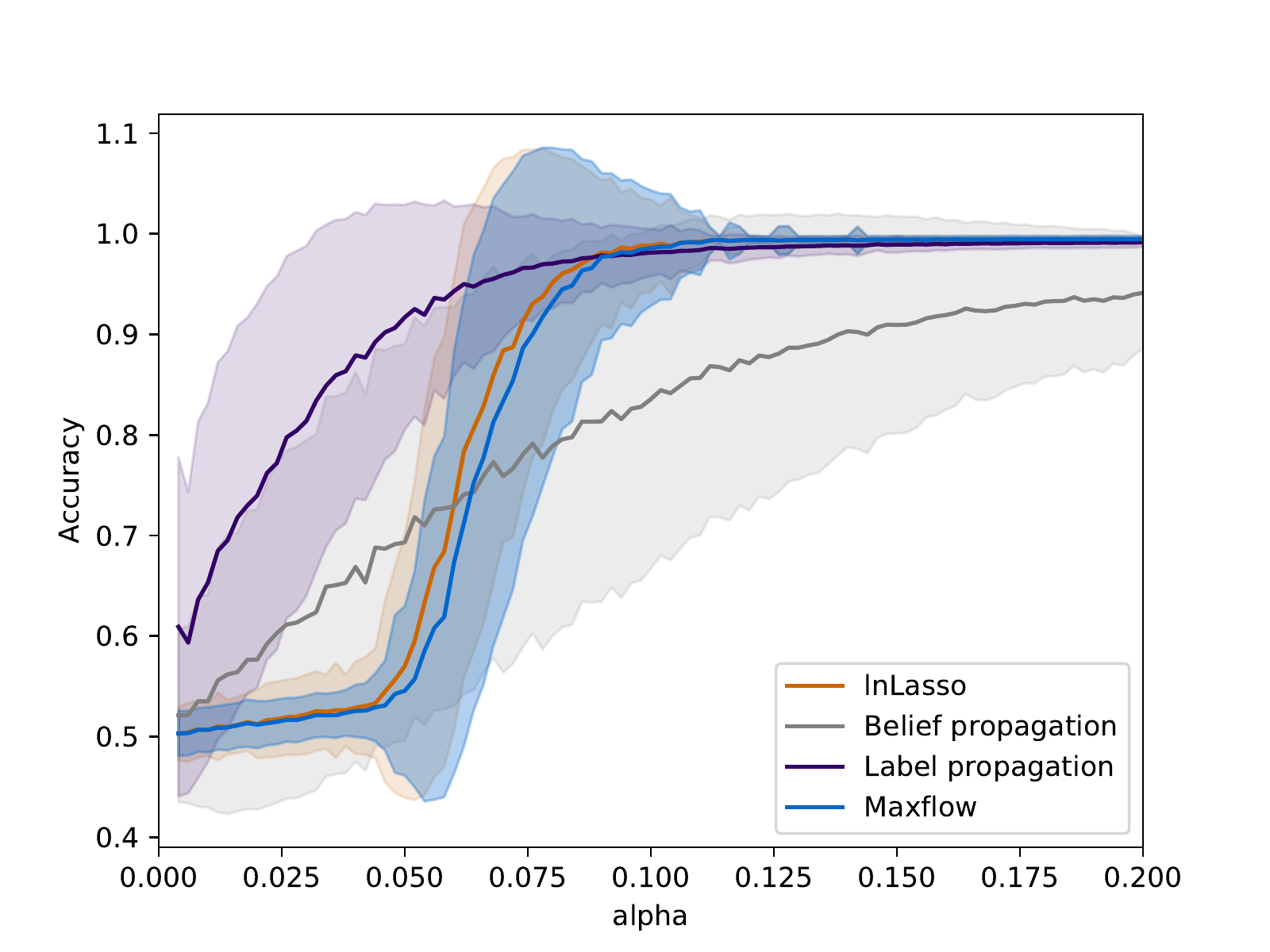}
		\vspace*{-3mm}
		\caption{Classification accuracy for varying sampling ratio $\alpha$.}
		\label{fig:alpha_test}
		\vspace*{-3mm}
	\end{figure}
	%It is remarkable how well plain LP performs compared to lnLasso for the SBM based datasets (see Fig.\ \ref{fig:epsilon_test} and Fig.\ \ref{fig:alpha_test}). 
	%We expect this to be due to the significantly different value ranges of edge weights used for the two clusters. 
	
	In a second experiment we compared LP with lnLasso on a dataset whose empirical graph is 
	a chain. This chain graph is partitioned into two clusters $\cluster_{1} = \{1,\ldots,50\}$ and $\cluster_{2} = \{51,\ldots,100\}$. 
	Edges within the same cluster have weight $W_{i,i+1}=1$, while the boundary edge $\{50,51\}$ has weight $W_{50,51}=1/2$. 
	The nodes are labeled as $y_{i} = -1 $ for $i \in \cluster_{1}$ and $y_{i}=1$ for $i \in \cluster_{2}$. We observe the labels only 
	for the sampling set $\samplingset = \{10,60\}$. 
	
	In Fig.\ \ref{scalar_lin_space}, we depict the resulting classifiers obtained from LP and lnLasso (Alg.\ \ref{alg:ADMM} with $\rho=1$ and $\lambda=1/10$), 
	each running for $1000$ iterations. In contrast to LP, lnLasso recovers the cluster structure perfectly. 
	%compared to LP which shows a smoothing effect. %The absolute values of the lnLasso classifier are significantly smaller 
	%compared to LP around the boundary edge $\{50,51\}$. 
	
	In a third experiment we use lnLasso to perform forground extraction on images. We represent a RGB image as a grid graph 
	with each node representing a particular pixel. The nodes are connected to their nearest neighbours. Each node is associated 
	a label $y=1$ if the corresponding pixel is foreground and $y=-1$ if the pixel is background. In Fig.\ \ref{fig:fg_tiger}, we show the result 
	by highlighting pixel $i$ according to the classifier value $x[i]$. 
		\begin{figure}[htbp]
		\includegraphics[width=0.9\columnwidth]{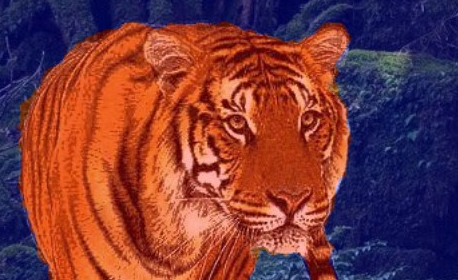}
		\vspace*{-3mm}
		\caption{The foreground extracted by lnLasso.}
		\label{fig:fg_tiger}
	\end{figure}
	
	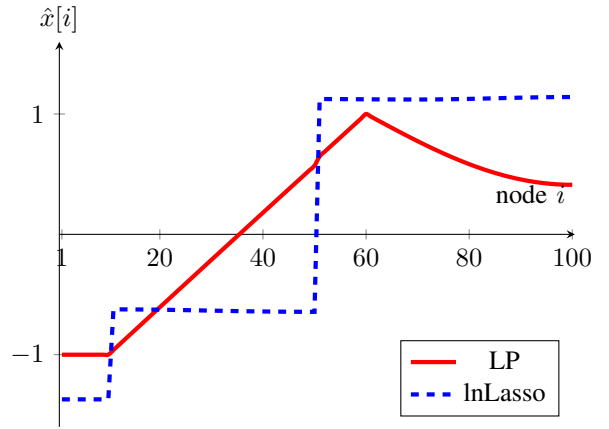
\begin{figure}[htbp]
\begin{center}
    \begin{tikzpicture}
          \begin{axis}
     [yscale=0.9,
      axis x line=center,
  axis y line=center,
 xtick={1,20,40,60,80,100}, 
  ytick={-1,1},
  xlabel={node $i$},
  ylabel={$\hat{x}[i]$},
  xlabel style={},
  ylabel style={above},
  xmin=0.5,
  xmax=100.5,
  ymin=-1.6,
  ymax=1.6,
  legend pos=south east
  ]
  \addplot  [smooth, color=red, ultra thick] table [x=a, y=b, col sep=comma] {lnLassoLP.csv} node [right,color=red] {LP};
    \addlegendentry{LP}; 
  \addplot  [color=blue, ultra thick,dashed] table [x=a, y=c, col sep=comma] {lnLassoLP.csv} node [right,color=blue] {lnLasso}; 
    \addlegendentry{lnLasso}
%  \addplot [color=black, ultra thick] table [x=a, y=d, col sep=comma] {linear.csv}node [right,color=black] {$h^{(0.7)}(x)\!=\!0.7x$}; 
 \end{axis}
%\node [right,color=red] at (8.3,1.5) {$h^{(\mathbf{w})} = 7(x^2 - x^3)$};
%\node [right,color=blue] at (2.7,0.5) {$h^{(\mathbf{w})} =x^2$};
%\node [right] at (3.5,2.2) {$h^{(\mathbf{w})}\!=\!x$};
     \end{tikzpicture}
     \vspace*{-3mm}
\end{center}
\caption{LP and lnLasso for chain graph.}
\label{scalar_lin_space}
\end{figure}

\vspace*{-1mm}
\section{Conclusions}
\vspace*{-1mm}

	We proposed the lnLasso for classifying network-structured datasets. 
	In contrast to LP, which uses the squared error loss, lnLasso uses the logistic loss to 
	within regularized empirical risk minimization. The regularization term of lnLasso is the 
	TV of the classifier requiring it to conform with the cluster structure of the empirical graph. 
	A scalable implementation of lnLasso is obtained via inexact ADMM. 
	The effectiveness of lnLasso to learn the labels of network-structured data is assessed by 
	means of illustrative numerical experiments. Our work opens several avenues for future research:  
         We plan to extend the current approach for binary classification to multi-class and multi-label 
	problems. Moreover, we aim at analysing the statistical properties of lnLasso. This analysis would help to guide the 
	choice of the regularization parameter in lnLasso.  

	\bibliographystyle{IEEEbib}
	\bibliography{SLPBib,tf-zentral}
	%\bibliography{}

\end{document}